\documentclass{article}

\usepackage[utf8]{inputenc}
\usepackage[T1]{fontenc}
\usepackage{textcomp}
\usepackage{bm}

\usepackage{amsmath}
\usepackage{amssymb}
\usepackage{amsthm}
\usepackage{amsfonts}       

\usepackage{mathtools}
\DeclarePairedDelimiter{\prn}{(}{)}
\DeclarePairedDelimiter{\set}{\{}{\}}
\DeclarePairedDelimiterX{\Set}[2]{\{}{\}}{\,{#1}\;\delimsize|\;{#2}\,}
\DeclarePairedDelimiter{\abs}{|}{|}
\DeclarePairedDelimiter{\norm}{\|}{\|}

\DeclarePairedDelimiter{\brc}{[}{]}
\DeclarePairedDelimiterX{\Brc}[2]{[}{]}{\,{#1}\;\delimsize|\;{#2}\,}

\DeclarePairedDelimiter{\floor}{\lfloor}{\rfloor}

\usepackage[svgnames]{xcolor}
\usepackage{graphicx}

\usepackage{booktabs}
\usepackage{multirow}

\usepackage{etoolbox}
\cslet{blx@noerroretextools}\empty%
\usepackage[date=year,eprint=false,doi=false,isbn=false,maxcitenames=2,natbib,url=false,sorting=nyt,style=trad-abbrv,backref=true]{biblatex}

\DeclareSourcemap{
    \maps[datatype=bibtex]{
        \map{
            \step[fieldset=editor, null]
            \step[fieldset=series, null]
            \step[fieldset=language, null]
            \step[fieldset=address, null]
            \step[fieldset=location, null]
            \step[fieldset=month, null]
        }
    }
}
\usepackage{algorithm}

\usepackage[colorlinks=true,citecolor=Navy,linkcolor=Maroon,urlcolor=Orchid,bookmarksnumbered,hypertexnames=false,pdfdisplaydoctitle,pdfusetitle,unicode]{hyperref}

\usepackage[noend]{algpseudocode}

\algnewcommand{\algorithmicinput}{\textbf{Input:}}
\algnewcommand{\Input}{\item[\algorithmicinput]}
\algnewcommand{\algorithmicoutput}{\textbf{Input:}}
\algnewcommand{\Output}{\item[\algorithmicoutput]}
\algnewcommand{\Break}{\textbf{break}}

\usepackage{url}

\usepackage{upref}
\usepackage[capitalize,noabbrev]{cleveref}

\crefname{step}{Step}{Steps}
\Crefname{step}{Step}{Steps}

\usepackage{autonum}

\usepackage{nicefrac}       
\usepackage{microtype}      
\usepackage{xspace}
\usepackage[color={red!100!green!33},colorinlistoftodos,prependcaption,textsize=small]{todonotes}
\usepackage{subcaption}
\usepackage{wrapfig}
\usepackage{siunitx}

\usepackage{fullpage}
\usepackage{newtxtext,newtxmath}

\newtheorem{theorem}{Theorem}[section]
\newtheorem{lemma}[theorem]{Lemma}

\newtheorem{assumption}[theorem]{Assumption}
\theoremstyle{definition}
\newtheorem{definition}[theorem]{Definition}

\newtheorem{remark}[theorem]{Remark}

\newcommand{\R}{\mathbb{R}}

\newcommand{\zeros}{\mathbf{0}}
\newcommand{\ones}{\mathbf{1}}

\DeclareMathOperator{\argmin}{arg\,min}
\DeclareMathOperator{\argmax}{arg\,max}

\newcommand{\Ord}{\mathrm{O}}

\newcommand{\A}{\bm{A}}
\newcommand{\B}{\bm{B}}
\renewcommand{\P}{\bm{P}}
\newcommand{\Q}{\bm{Q}}
\newcommand{\V}{\bm{V}}
\newcommand{\X}{\bm{X}}

\let\origc\c
\renewcommand{\b}{\bm{b}}
\renewcommand{\c}{\bm{c}}

\newcommand{\w}{\bm{w}}
\newcommand{\x}{\bm{x}}
\newcommand{\y}{\bm{y}}
\newcommand{\z}{\bm{z}}

\newcommand{\lamb}{\bm{\lambda}}
\newcommand{\thb}{\bm{\theta}}

\newcommand{\itt}{r}
\newcommand{\jtt}{s}
\newcommand{\e}{\mathbf{e}}
\newcommand{\I}{\mathbf{I}}

\newcommand{\Dcal}{\mathcal{D}}
\newcommand{\Ucal}{\mathcal{U}}
\newcommand{\pdim}{\mathrm{pdim}}
\newcommand{\E}{\mathop{\mathbb{E}}}

\newcommand{\Full}{\textsf{Full}\xspace}
\newcommand{\ColRand}{\textsf{ColRand}\xspace}
\newcommand{\PCA}{\textsf{PCA}\xspace}
\newcommand{\SGA}{\textsf{SGA}\xspace}

\title{Generalization Bound and Learning Methods for \\Data-Driven Projections in Linear Programming}

\author{%
Shinsaku Sakaue\\
The University of Tokyo\\
Tokyo, Japan\\
\href{mailto:sakaue@mist.i.u-tokyo.ac.jp}{sakaue@mist.i.u-tokyo.ac.jp}
\and
Taihei Oki\\
The University of Tokyo\\
Tokyo, Japan\\
\href{mailto:oki@mist.i.u-tokyo.ac.jp}{oki@mist.i.u-tokyo.ac.jp}
}
\date{}

\begin{document}

\maketitle

\begin{abstract}  
  How to solve high-dimensional linear programs (LPs) efficiently is a fundamental question.
  Recently, there has been a surge of interest in reducing LP sizes using \textit{random projections}, which can accelerate solving LPs independently of improving LP solvers. 
  This paper explores a new direction of \emph{data-driven projections}, which use projection matrices learned from data instead of random projection matrices.
  Given training data of $n$-dimensional LPs, we learn an $n\times k$ projection matrix with $n > k$. 
  When addressing a future LP instance, we reduce its dimensionality from $n$ to $k$ via the learned projection matrix, solve the resulting LP to obtain a $k$-dimensional solution, and apply the learned matrix to it to recover an $n$-dimensional solution.
  On the theoretical side, a natural question is: how much data is sufficient to ensure the quality of recovered solutions? We address this question based on the framework of \textit{data-driven algorithm design}, which connects the amount of data sufficient for establishing generalization bounds to the \textit{pseudo-dimension} of performance metrics. We obtain an $\tilde{\mathrm{O}}(nk^2)$ upper bound on the pseudo-dimension, where $\tilde{\mathrm{O}}$ compresses logarithmic factors. We also provide an $\Omega(nk)$ lower bound, implying our result is tight up to an $\tilde{\mathrm{O}}(k)$ factor. 
  On the practical side, we explore two simple methods for learning projection matrices: PCA- and gradient-based methods. While the former is relatively efficient, the latter can sometimes achieve better solution quality. Experiments demonstrate that learning projection matrices from data is indeed beneficial: it leads to significantly higher solution quality than the existing random projection while greatly reducing the time for solving LPs.
\end{abstract}

\section{Introduction}\label{sec:introduction}
Linear programming (LP) has been one of the most fundamental tools used in various industrial domains \citep{Gass1985-dz,Eiselt2007-an}, and how to address high-dimensional LPs efficiently has been a major research subject in operations research.
To date, researchers have developed various fast LP solvers, most of which stem from the simplex or interior-point method. 
Recent advances include a parallelized simplex method~\citep{Huangfu2018-ng} and a randomized interior-point method~\citep{Chowdhury2022-xy}.
%
Besides the improvements in LP solvers, there has been a growing interest in reducing LP sizes via \textit{random projections} \citep{Vu2018-ty,Poirion2023-bm,Akchen2024-gg}, motivated by the success of random sketching in numerical linear algebra \citep{Woodruff2014-ya}. 
Such a projection-based approach is \textit{solver-agnostic} in that it can work with any solvers, including the aforementioned recent solvers, for solving reduced-size LPs.
This solver-agnostic nature is beneficial, especially considering that LP solvers have evolved in distinct directions of simplex and interior-point methods. 

In the context of numerical linear algebra, there has been a notable shift towards learning sketching matrices from data, instead of using random matrices \citep{Indyk2019-cn,Indyk2021-yn,Bartlett2022-mu,Li2023-sp,Sakaue2023-ta}. 
This data-driven approach is effective when we frequently address similar instances.
The line of previous research has demonstrated that learned sketching matrices can greatly improve the performance of sketching-based methods.

\subsection{Our contribution}
Drawing inspiration from this background, we study a \textit{data-driven projection} approach for accelerating repetitive solving of similar LP instances, which often arise in practice \citep{Fan2023-gr} (see also \cref{rem:validity}). 
Our approach inherits the solver-agnostic nature of random projections for LPs, and it can improve solution quality by learning projection matrices from past LP instances. 
Our contribution is a cohesive study of this data-driven approach to LPs from both theoretical and practical perspectives, as follows.

\textbf{Generalization bound.}
We first formulate the task of learning projection matrices as a statistical learning problem and study the generalization bound. 
Specifically, we analyze the number of LP-instance samples sufficient for bounding the gap between the empirical and expected objective values attained by the data-driven projection approach. 
Such a generalization bound is known to depend on the \textit{pseudo-dimension} of the class of performance metrics. 
We prove an $\tilde{\mathrm{O}}(nk^2)$ upper bound on the pseudo-dimension (\cref{theorem:pdim-upper-bound}), where $n$ and $k$ are the original and reduced dimensionalities, respectively, and $\tilde{\mathrm{O}}$ compresses logarithmic factors. 
A main technical non-triviality lies in \cref{lem:lp-polynomial}, which elucidates a piecewise polynomial structure of the optimal value of LPs as a function of input parameters.
Besides playing a key role in proving \cref{theorem:pdim-upper-bound}, \cref{lem:lp-polynomial} offers general insight into the optimal value of LPs, which could have broader implications.
We also give an $\Omega(nk)$ lower bound on the pseudo-dimension (\cref{theorem:pdim-lower-bound}). 
As experiments demonstrate later, we can get high-quality solutions with $k$ much smaller than $n$, suggesting our result, with only an $\tilde{\mathrm{O}}(k)$ gap, is nearly tight.

\textbf{Learning methods.} We then explore how to learn projection matrices in \cref{sec:learning-methods}.
We consider two simple learning methods based on principal component analysis (PCA) and gradient updates.
The former efficiently constructs a projection matrix by extracting the top-$k$ subspace around which optimal solutions of future instances are expected to appear. 
The latter, although more costly, directly improves the optimal value of LPs via gradient ascent.
In \cref{section:experiment}, experiments on various datasets confirm that projection matrices learned by the PCA- and gradient-based methods can lead to much higher solution quality than random projection \citep{Akchen2024-gg}, while greatly reducing the time for solving LPs.

\subsection{Related work}
\textbf{Random projections for LPs.}
\citet{Vu2018-ty} introduced a random-projection method to reduce the number of equality constraints, and \citet{Poirion2023-bm} extended it to inequality constraints. 
As discussed therein, reducing the number of inequality constraints of LPs corresponds to reducing the dimensionality (the number of variables) of dual LPs. 
Recently, \citet{Akchen2024-gg} developed a column-randomized method for reducing the dimensionality of LPs. 
While these studies provide high probability guarantees, we focus on data-driven projections and discuss generalization bounds.

\textbf{Data-driven algorithm design.}
\textit{Data-driven algorithm design} \citep{Balcan2021-fy}, initiated by \citet{Gupta2017-ng}, has served as a foundational framework for analyzing generalization bounds of various data-driven algorithms~\citep{Balcan2021-jv,Balcan2022-em,Balcan2022-yp,Bartlett2022-mu,Sakaue2022-bb,Balcan2023-dn}.
Our statistical learning formulation in \cref{subsec:data-driven-projection} and a general proof idea in \cref{sec:generalization-bound} are based on it. 
Among the line of studies, the analysis technique for data-driven integer-programming (IP) methods~\citep{Balcan2022-em,Balcan2022-yp} is close to ours. 
The difference is that while their technique is intended for analyzing IP methods (particularly, branch-and-cut methods), we focus on LPs and discuss a general property of the optimal value viewed as a function of input parameters. 
Thus, our analysis is independent of solution methods, unlike the previous studies. 
This aspect is crucial for analyzing our solver-agnostic approach. 
Some studies have also combined LP/IP methods with machine learning \citep{Berthold2021-wk,Fan2023-gr,Sun2023-yc}, while learning of projection matrices has yet to be studied.

\textbf{Learning through optimization.}
Our work is also related to the broad stream of research on learning through optimization procedures~\citep{Amos2017-sf,Wilder2019-yw,Agrawal2019-sv,Berthet2020-wm,Tan2020-lp,Meng2021-yb,Amos2023-kw,Elmachtoub2022-og,Wang2020-wt,El_Balghiti2023-xe}, which we discuss in \cref{app:related-work}.

\textbf{Notation.}
For a positive integer $n$, let $\I_{n}$ and $\zeros_{n}$ be the $n\times n$ identity matrix and the $n$-dimensional all-zero vector, respectively, where we omit the subscript when it is clear from the context. 
For two matrices $\A$ and $\B$ with the same number of rows (columns), $[\A, \B]$ ($[\A; \B]$) denotes the matrix obtained by horizontally (vertically) concatenating $\A$ and $\B$.

\section{Reducing dimensionality of LPs via projection}\label{sec:projection-for-lps}
We overview the projection-based approach for reducing the dimensionality of LPs~\citep{Poirion2023-bm,Akchen2024-gg}. 
For ease of dealing with feasibility issues, we focus on the following inequality-form LP with input parameters $\c\in\R^n$, $\A\in\R^{m\times n}$, and $\b\in\R^m$: 
\begin{equation}
  \begin{aligned}\label{prob:lp}
    {\text{maximize}}_{\x\in\R^n} \quad \c^\top \x 
    && &&  
    \mathop{\text{subject to}}\quad \A\x \le \b. 
  \end{aligned}
\end{equation}
When $n$ is large, restricting variables to a low-dimensional subspace can be helpful for computing an approximate solution to \eqref{prob:lp} quickly.
Specifically, given a \textit{projection matrix} $\P\in\R^{n\times k}$ with $n > k$, we consider solving the following \textit{projected LP}, instead of \eqref{prob:lp}: 
\begin{equation}
  \begin{aligned}\label{prob:reduced-lp}
    {\text{maximize}}_{\y\in\R^k} \quad \c^\top \P\y 
    && &&  
    \mathop{\text{subject to}}\quad \A\P\y \le \b.  
  \end{aligned}
\end{equation}
Once we get an optimal solution $\y^*$ to the projected LP \eqref{prob:reduced-lp}, we can recover an $n$-dimensional solution, $\tilde{\x} = \P\y^*$, to the original LP~\eqref{prob:lp}. 
Note that the recovered solution is always feasible for \eqref{prob:lp}, although not always optimal. 
We measure the solution quality with the objective value $\c^\top\tilde{\x} = \c^\top \P\y^*$.
Ideally, if $\P$'s columns span a linear subspace that contains an optimal solution to \eqref{prob:lp}, the recovered solution $\tilde{\x} = \P\y^*$ is optimal to~\eqref{prob:lp} due to the optimality of $\y^*$ to \eqref{prob:reduced-lp}. 
Therefore, if we find such a good $\P$ close to being ideal with small $k$, we can efficiently obtain a high-quality solution $\tilde{\x} = \P\y^*$ to \eqref{prob:lp} by solving the smaller projected LP \eqref{prob:reduced-lp}.

\begin{remark}[Solver-specific aspects]\label{rem:solver-specific}
  As mentioned in \cref{sec:introduction}, this projection-based approach is solver-agnostic in that we can apply any LP solver to projected LPs~\eqref{prob:reduced-lp}. 
  To preserve this nature, we focus on designing projection matrices and do not delve into solver-specific discussions. 
  Experiments in \cref{section:experiment} will use Gurobi as a fixed LP solver, which is a standard choice. 
  Strictly speaking, projections alter the sparsity and numerical stability of projected LPs, which can affect the performance of solvers. 
  This point can be important, especially when original LPs are sparse and solvers exploit the sparsity. 
  Investigating how to take such solver-specific aspects into account is left for future work.
\end{remark}

\section{Data-driven projection}\label{subsec:data-driven-projection}
While the previous studies \citep{Vu2018-ty,Poirion2023-bm,Akchen2024-gg} have reduced LP sizes via random projections, we may be able to improve solution quality by learning projection matrices from data. 
We formalize this idea as a statistical learning problem.
Let $\Pi$ denote the set of all possible LP instances and $\Dcal$ an unknown distribution on $\Pi$. 
Given LP instances sampled from $\Dcal$, our goal is to learn $\P$ that maximizes the expected optimal value of projected LPs over $\Dcal$. 
Below, we assume the following three conditions.
\begin{assumption}\label{assump:feasible_bounded}
  (i) Every $\pi \in \Pi$ takes the inequality form \eqref{prob:lp}, 
  (ii) $\x = \zeros_n$ is feasible for all $\pi \in \Pi$, and 
  (iii) optimal values of all instances in $\Pi$ are upper bounded by a finite constant $H > 0$. 
\end{assumption}
Although \Cref{assump:feasible_bounded} narrows the class of LPs we can handle, it is not as restrictive as it seems. 
Suppose for example that LP instances in $\Pi$ have identical equality constraints. 
While such LPs in their current form do not satisfy~(i), we can convert them into the inequality form~\eqref{prob:lp} by considering the null space of the equality constraints (see \cref{app:remove-eq} for details), hence satisfying~(i). 
This conversion is useful for dealing with LPs of maximum-flow and minimum-cost-flow problems on a fixed graph topology, where we can remove the flow-conservation equality constraints by considering the null space of the incidence matrix of the graph. 
Regarding condition~(ii), we may instead assume that there exists an arbitrary common feasible solution $\x_0$ without loss of generality. 
This is because we can translate the feasible region so that $\x_0$ coincides with the origin $\zeros_n$. 
Condition~(ii) also implies that for any $\P\in\R^{n\times k}$, projected LPs are \emph{feasible} (i.e., their feasible regions are non-empty) since $\y = \zeros_k$ is always feasible for any projected LPs.
Condition~(iii) is satisfied simply by focusing on \emph{bounded} LPs (i.e., LPs with finite optimal values) and setting $H$ to the largest possible optimal value.
Conditions~(ii) and~(iii) also ensure that the optimal value of projected LPs always lies in $[0, H]$, which is used to derive a generalization bound in \cref{sec:generalization-bound}.
In \cref{section:experiment}, we see that many problems, including packing and network flow problems, can be written as LPs satisfying \cref{assump:feasible_bounded}. 

Due to condition (i), we can identify each LP instance $\pi \in \Pi$ with its input parameters $(\c, \A, \b)$ in~\eqref{prob:lp}, i.e., $\pi = (\c, \A, \b)$.
For an LP instance $\pi\in\Pi$ and a projection matrix $\P \in \R^{n\times k}$, we define 
\begin{equation}\label{eq:u-definition}
  u(\P, \pi) = \max\Set{\c^\top\P\y}{\A\P\y\le \b}
\end{equation}
as the optimal value of the projected LP. 
Our goal is to learn $\P\in\R^{n\times k}$ from LP instances sampled from $\Dcal$ to maximize the expected optimal value on future instances, i.e., $\E_{\pi\sim\Dcal}[u(\P, \pi)]$.

\begin{remark}[Validity of the setting]\label{rem:validity}
  The above statistical learning setting regarding LP instances is not an artifact.
  As \citet{Fan2023-gr} discussed, LPs often serve as descriptive models, and each instance can be viewed as a realization of input parameters following some distribution. 
  Such a scenario arises in, for example, daily production planning and flight scheduling. 
  Note that the statistical learning setting is also widely used as a foundational framework in data-driven algorithm design \citep{Gupta2017-ng,Balcan2021-jv,Bartlett2022-mu,Balcan2023-dn}. 
\end{remark}

\section{Generalization bound}\label{sec:generalization-bound}
This section studies the generalization bound, namely, how many samples from $\Dcal$ are sufficient for guaranteeing that the expected optimal value, $\E_{\pi\sim\Dcal}[u(\P, \pi)]$, of learned $\P$ is close to the empirical optimal value on sampled instances. 
First, let us overview the basics of learning theory. 
Let $\Ucal\subseteq\R^\Pi$ be a class of functions, where each $u\in\Ucal$ takes some input $\pi\in\Pi$ and returns a real value. 
We use the following \textit{pseudo-dimension}~\citep{Pollard1984-zp} to measure the complexity of a class of real-valued functions.
\begin{definition}\label{def:pdim}
  Let $N$ be a positive integer.
  We say $\Ucal\subseteq\R^\Pi$ \textit{shatters} an input set, $\set{\pi_1,\dots,\pi_N}\subseteq \Pi$, if there exist threshold values, $t_1,\dots,t_N \in \R$, such that each of all the $2^N$ outcomes of $\Set{u(\pi_i) \ge t_i}{i=1,\dots,N}$  is realized by some $u\in\Ucal$.
  The \textit{pseudo-dimension} of $\Ucal$, denoted by $\pdim(\Ucal)$, is the maximum size of an input set that $\Ucal$ can shatter.
\end{definition}

In our case, the set $\Ucal$ consists of functions $u(\P, \cdot):\Pi\to\R$, defined in \eqref{eq:u-definition}, for all possible projection matrices $\P\in\R^{n\times k}$. 
Each $u(\P, \cdot) \in \Ucal$ takes an LP instance $\pi = (\c, \A, \b) \in \Pi$ as input and returns the optimal value of the projected LP. 
\Cref{assump:feasible_bounded} ensures that the range of $u(\P, \cdot)$ is bounded by $[0, H]$ for all $\P\in\R^{n\times k}$.
Thus, the well-known uniform convergence result (see, e.g., \citet[Theorems~17.7~and~18.4]{Anthony2009-mm})
implies that for any distribution $\Dcal$ on $\Pi$, $\varepsilon>0$, and $\delta\in(0,1)$, if $N = \Omega((H/\varepsilon)^{2}(\pdim(\Ucal)\cdot\log(H/\varepsilon) + \log(1/\delta)))$ instances drawn i.i.d.\ from $\Dcal$ are given, with probability at least $1-\delta$, for all $\P\in\R^{n\times k}$, it holds that
\begin{equation}\label{eq:generalization-bound} 
  \textstyle
  \abs*{\frac{1}{N}
  \sum_{i=1}^N u(\P, \pi_i) - \E_{\pi\sim\Dcal}\brc*{u(\P, \pi)}} 
  \le \varepsilon.
\end{equation}
That is, if a projection matrix $\P$ produces high-quality solutions on $N \approx (H/\varepsilon)^2 \cdot \pdim(\Ucal)$ instances sampled i.i.d.\ from $\Dcal$, it likely yields high-quality solutions on future instances from $\Dcal$ as well.
Thus, analyzing $\pdim(\Ucal)$ of $\Ucal = \Set{u(\P, \cdot):\Pi\to\R}{\P \in \R^{n\times k}}$ reveals the sufficient sample size.

\begin{remark}[Importance of uniform convergence]\label{rem:uniform-bound}
  While the above generalization bound is not the sole focus of learning theory, it is particularly valuable in data-driven algorithm design, as is also discussed in the literature \citep{Balcan2021-jv,Balcan2022-em}.
  Note that \eqref{eq:generalization-bound} holds uniformly for all $\P \in \R^{n\times k}$, offering performance guarantees \emph{regardless of how $\P$ is learned}.
  Thus, we may select learning methods based on their empirical performance. 
  This is helpful since there are no gold-standard methods for learning parameters of algorithms; 
  we discuss learning methods for our case in \cref{sec:learning-methods}.
  This situation differs from the standard supervised learning setting, where we minimize common losses, e.g., squared and logistic.
  Additionally, the uniform bound ensures that learned $\P$ does not overfit sampled instances.
\end{remark}


\subsection{Upper bound on $\pdim(\Ucal)$}
Building upon the above learning theory background, a crucial factor for establishing the generalization bound is $\pdim(\Ucal)$.
To upper bound this, we give a structural observation of the optimal value of LPs (\cref{lem:lp-polynomial}) and combine it with a general proof idea in data-driven algorithm design \citep{Gupta2017-ng,Balcan2021-fy}.

We first overview the general proof idea.
Suppose that we have an upper bound on the number of outcomes of $\Set{u(\P, \pi_i) \ge t_i}{i=1,\dots,N}$ that grows more slowly than $2^N$.
Since shattering $N$ instances requires $2^N$ outcomes, the largest $N$, such that the upper bound is at least $2^N$, serves as an upper bound on $\pdim(\Ucal)$ (intuitively, $\pdim(\Ucal) \lesssim \log_2(\text{``upper bound on the number of outcomes''})$).
Below, we discuss bounding the number of outcomes, which is the most technically important step.

To examine the number of possible outcomes, we consider a fundamental question related to sensitivity analysis of LPs: \textit{how does the optimal value of an LP behave when input parameters change?}\footnote{While a similar question is studied in \citet[Theorem~3.1]{Balcan2022-yp}, their result focuses on the case where new constraints are added to LPs to analyze branch-and-cut methods, unlike our \cref{lem:lp-polynomial}.} 
In our case, a projected LP has input parameters $(\P^\top\c, \A\P, \b) \in \R^k\times \R^{m\times k} \times \R^m$, where $\P^\top \c$ and $\A\P$ change with $\P  \in \R^{n\times k}$. 
Thus, addressing this question offers insight into the number of outcomes.
\Cref{lem:lp-polynomial} provides an answer for a more general setting, which might find other applications beyond our case since learning through LPs is not limited to the projection-based approach \citep{Wilder2019-yw,Berthet2020-wm,Tan2020-lp,Elmachtoub2022-og}.

\begin{lemma}\label{lem:lp-polynomial}
  Let $t \in \R$ be a threshold value.  
  Consider an LP $\tilde{\pi} = (\tilde{\c}, \tilde{\A}, \tilde{\b}) \in \R^k\times \R^{m\times k} \times \R^m$ such that each entry of $\tilde{\c}$, $\tilde{\A}$, and $\tilde{\b}$ is a polynomial of degree at most $d$ in $\nu$ real variables, $\thb \in \R^\nu$. 
  Assume that $\tilde{\pi}$ is bounded and feasible for every $\thb \in \R^\nu$. 
  Then, there are up to $\binom{m+2k}{2k}(m+2k+2)$ polynomials of degree at most $(2k+1)d$ in $\thb$ whose sign patterns ($<0$, $=0$, or $>0$)
  partition $\R^\nu$ into some regions, and whether $\max\Set{\tilde{\c}^\top \y}{\tilde{\A}\y \le \tilde{\b}} \ge t$ or not is identical within each region.  
\end{lemma}
\newcommand{\Ap}{{\A'}}
\newcommand{\bp}{{\b'}}
\newcommand{\cp}{\c'}
\begin{proof}
  First, we rewrite the LP $\tilde{\pi} = (\tilde{\c}, \tilde{\A}, \tilde{\b})$ as an equivalent $2k$-dimensional LP with non-negativity constraints:
  $\max\Set[\large]{\tilde{\c}^\top(\y^+ - \y^-)}{\tilde{\A}(\y^+ - \y^-)\le\tilde{\b},\, [\y^+;\y^-]\ge\zeros}$.
  The resulting constraint matrix, $\Ap \coloneqq [\tilde{\A}, -\tilde{\A}; -\I_{2k}]$, has full column rank, which simplifies the subsequent discussion.
  Note that the maximum degree of input parameters remains at most $d$, while the sizes, $m$ and $k$, increase to $m'\coloneqq m+2k$ and $k'\coloneqq 2k$, respectively.
  Below, we focus on the reformulated LP $(\cp, \Ap, \bp) \in \R^{k'}\times \R^{m'\times k'} \times \R^{m'}$, where $\cp\coloneqq [\tilde{\c}; -\tilde{\c}]$, $\bp\coloneqq [\tilde{\b}; \zeros_{2k}]$, and $\Ap$ has full column rank.
  
  We consider determining $\max\Set{\cp^\top \y}{\Ap\y \le \bp} \ge t$ or not by checking all vertices of the feasible region. 
  For any size-$k'$ subset, $I \subseteq \set*{1,\dots,m'}$, of row indices of $\Ap \in \R^{m'\times k'}$, let $\Ap_I$ denote the $k'\times k'$ submatrix of $\Ap$ with rows restricted to $I$ and $\bp_I \in \R^{k'}$ the corresponding subvector of $\bp$. 
  For every subset $I$ with $\det \Ap_I \neq 0$, let $\y_I \coloneqq {\Ap_I}^{-1}\bp_I$. 
  Since the LP is bounded and feasible, and $\Ap$ has full column rank, there is a vertex optimal solution written as $\y_I = {\Ap_I}^{-1}\bp_I$ for some $I$ (see the proof of \citet[Proposition~3.1]{Korte2012-gy}).  
  Thus, the optimal value is at least $t$ if and only if there exists at least one size-$k'$ subset $I$ with 
  $\det \Ap_I \neq 0$, $\Ap \y_I \le \bp$, and $\cp^\top \y_I \ge t$.

  Based on the above observation, we identify polynomials whose sign patterns determine $\max\Set{\cp^\top \y}{\Ap\y \le \bp} \ge t$ or not. 
  For any subset $I$, if $\det \Ap_I \neq 0$, Cramer's rule implies that $\y_I = {\Ap_I}^{-1}\bp_I$ is written as $\bm{f}_I(\thb)/\det \Ap_I$, where $\bm{f}_I(\thb)$ is some $k'$-valued polynomial vector of $\thb$ with degrees at most $k'd$. 
  Thus, we can check $\Ap \y_I \le \bp$ and $\cp^\top \y_I \ge t$ by examining sign patterns of $m'+1$ polynomials, $\Ap\bm{f}_I(\theta) - (\det \Ap_I) \bp$ and $\cp^\top \bm{f}_I(\theta) - t\det \Ap_I$, whose degrees are at most $(k'+1)d$. 
  Considering all the $\binom{m'}{k'}$ choices of $I$, there are $\binom{m'}{k'}(m'+2)$ polynomials of the form $\det \Ap_I$, $\Ap\bm{f}_I(\theta) - (\det \Ap_I) \bp$, and $\cp^\top \bm{f}_I(\theta) - t\det \Ap_I$ with degrees at most $(k'+1)d$ such that their sign patterns partition $\R^\nu$ into some regions, and $\max\Set{\cp^\top \y}{\Ap\y \le \bp} \ge t$ or not is identical within each region. 
  Substituting $m+2k$ and $2k$ into $m'$ and $k'$, respectively, completes the proof.
\end{proof}

\Cref{lem:lp-polynomial} states that the outcome of whether $u(\P, \pi) = \max\Set{\c^\top\P \y}{\A\P\y \le \b}$ exceeds $t$ or not is determined by sign patterns of polynomials of $\P$, and an upper bound on the sign patters of polynomials is known as \textit{Warren's theorem} \citep{Warren1968-hp}, as detailed shortly.
Combining them with the aforementioned general idea yields the following upper bound on $\pdim(\Ucal)$.
\begin{theorem}\label{theorem:pdim-upper-bound}
  $\pdim(\Ucal) = \Ord(nk^2\log mk)$.
\end{theorem}
\begin{proof}
  Let $(\pi, t) \in \Pi\times\R$ be a pair of an LP instance and a threshold value.
  Setting $\thb = \P$ and $d = 1$ in \cref{lem:lp-polynomial}, we have up to $\binom{m+2k}{k}(m+2k+2)$ polynomials of degree at most $2k+1$ whose sign patterns determine whether $u(\P, \pi) \ge t$ or not. 
  Thus, given $N$ pairs of input instances and threshold values, $(\pi_i, t_i)_{i=1}^N$, we have up to $N \times \binom{m+2k}{2k}(m+2k+2)$ 
  polynomials whose sign patterns determine $u(\P, \pi_i) \ge t_i$ or not for all $i=1,\dots,N$, i.e., outcomes of $N$ instances. 
  
  Warren's theorem states that given $\ell$ polynomials of $\nu$ variables with degrees at most $\Delta$, the number of all possible sign patterns is at most $\prn*{8\mathrm{e}\ell\Delta/\nu}^\nu$ \citep{Warren1968-hp} (see also \citet[Corollary~2.1]{Goldberg1995-af}).
  In our case, the number of polynomials is $\ell = N \times \binom{m+2k}{2k}(m+2k+2)$, and each of them has $\nu = nk$ variables ($\P$'s entries) and degrees at most $\Delta = 2k+1$. 
  Thus, the number of all possible outcomes is at most 
  $\prn*{8\mathrm{e}N\binom{m+2k}{2k}\frac{(m+2k+2)(2k+1)}{nk}}^{nk} 
  \lesssim \prn*{\frac{N}{nk}}^{nk}\mathrm{poly}(m,k)^{nk^2}$.
  To shatter the set of $N$ instances, the right-hand side must be at least $2^N$. 
  Taking the base-2 logarithm, it must hold that 
  $
  N \lesssim nk\log_2\frac{N}{nk} + \Ord(nk^2\log mk) \le \frac23N + \Ord(nk^2\log mk)
  $, 
  where we used $x\log_2\frac{1}{x} \le \frac23$ for $x>0$.
  Therefore, $\Ucal$ can shatter $\Ord(nk^2\log mk)$ instances, obtaining the desired bound on $\pdim(\Ucal)$.
\end{proof}

\subsection{Lower bound on $\pdim(\Ucal)$}
We then provide an $\Omega(nk)$ lower bound on $\pdim(\Ucal)$ to complement the above $\tilde\Ord(nk^2)$ upper bound, implying the tightness up to an $\tilde\Ord(k)$ factor. 
See \cref{app:proof-lower-bound} for the proof.
\begin{theorem}\label{theorem:pdim-lower-bound}
  $\pdim(\Ucal) = \Omega(nk)$.
\end{theorem}
Our proof indeed gives the same lower bound on the \textit{$\gamma$-fat shattering dimension} for $\gamma < 1/2$, which implies a lower bound of $\Omega(nk/\varepsilon)$ on $N$, the sample size needed to guarantee~\eqref{eq:generalization-bound} \citep[Theorem~19.5]{Anthony2009-mm}. 
Thus, in terms of the sample complexity, our result is tight up to an $\tilde{\Ord}(k/\varepsilon)$ factor. 
The $1/\varepsilon$ gap is inevitable in general \citep[Section~19.5]{Anthony2009-mm}, while closing the $\tilde{\Ord}(k)$ gap is an interesting open problem.


\section{Learning methods}\label{sec:learning-methods}
We then discuss how to learn projection matrices from training datasets. 
From the bound~\eqref{eq:generalization-bound}, given $N$ training LP instances, the expected solution quality on future instances likely remains within the range of $\pm\varepsilon$ from the empirical one, where $\varepsilon \lesssim H\sqrt{\pdim(\Ucal)/N} \lesssim Hk\sqrt{n/N}$ due to \cref{theorem:pdim-upper-bound}, \textit{regardless of how we learn a projection matrix $\P$}.
Therefore, in practice, we only need to find an empirically good projection matrix $\P$, which motivates us to explore various ideas for learning $\P$.
Below, we discuss two natural ideas: PCA- and gradient-based methods.

\begin{remark}[Training time]\label{rem:learning-time}
We emphasize that learning methods are used only before addressing future LP instances and not once a projection matrix $\P$ is learned.
Hence, they can take much longer than the time for solving new LP instances. 
Similarly, we suppose that optimal solutions to training instances are available, as we can compute them a priori.
Note that similar premises are common in most data-driven algorithm research \citep{Indyk2019-cn,Berthold2021-wk,Bartlett2022-mu,Balcan2023-dn,Fan2023-gr,Sun2023-yc}.
Considering this, our learning methods are primarily intended for conceptual simplicity, not for efficiency. 
For completeness, we present the theoretical time complexity and the training time taken in the experiments in \cref{app:time-learning}.
\end{remark}

\subsection{PCA-based method}\label{subsec:pca}
As described in \cref{sec:projection-for-lps}, a projection matrix $\P$ should preferably have columns that span a low-dimensional subspace around which future optimal solutions will appear.
Hence, a natural idea is to use PCA to extract such a subspace, regarding optimal solutions to training instances as data points. 

Formally, let $\X \in \R^{N \times n}$ be a matrix whose $i$th row is an optimal solution to the $i$th training instance. 
We apply PCA to this $\X$. 
Specifically, we subtract the mean, $\bar\x = \frac{1}{N}\X^\top\ones_N$, from each row of $\X$ and apply the singular value decomposition (SVD) to $\X - \ones_N \bar\x^\top$, obtaining a decomposition of the form $\bm{U}\bm{\Sigma}\bm{V}^\top = \X - \ones_N \bar\x^\top$. 
Let $\V_{k-1} \in \R^{n\times (k-1)}$ be the submatrix of $\V$ whose columns are the top-($k-1$) right-singular vectors of $\X - \ones_N \bar\x^\top$.
We use $\P = [\bar\x, \V_{k-1}] \in \R^{n \times k}$ as a projection matrix. 
Here, $\bar\x$ is concatenated due to the following consideration: 
since $\V_{k-1}$ is designed to satisfy $\V_{k-1}\bm{Y}'\approx \X^\top - \bar\x \ones_N^\top$ for some $\bm{Y}' \in \R^{(k-1)\times N}$, we expect $[\bar\x, \V_{k-1}]\bm{Y} \approx \X^\top$ to hold for some $\bm{Y} \in \R^{k\times N}$, hence $\P = [\bar\x, \V_{k-1}]$.
This method is not so costly when optimal solutions to training LP instances are given, as it only requires finding the top-($k-1$) right-singular vectors of $\X - \ones_N \bar\x^\top$.

\subsection{Gradient-based method}\label{subsec:gradient}
While the PCA-based method aims to extract the subspace into which future optimal solutions are likely to fall, it only uses optimal solutions and discards input parameters of LPs.
As a complementary approach, we provide a gradient-based method that directly improves the optimal value of LPs.

As a warm-up, consider maximizing $u(\P, \pi) = \max\Set{\c^\top\P\y}{\A\P\y\le\b}$ of a single LP instance $\pi = (\c, \A, \b)$ via gradient ascent.
Assume that the projected LP satisfies a \textit{regularity condition}, which requires the existence of an optimal solution $\y^*$ at which active constraints are linearly independent. 
Then, $u(\P, \pi)$ is differentiable in $\P$ and the gradient is expressed as follows \citep[Theorem~1]{Tan2020-lp}
(see \cref{app:gradient} for details of the derivation):
\begin{equation}\label{eq:gradient}
  \nabla u(\P, \pi) = \c {\y^*}^\top - \A^\top \lamb^* {\y^*}^\top,
\end{equation}
where $\lamb^* \in \R^m$ is a dual optimal solution. 
Thus, we can use the gradient ascent method to maximize $u(\P, \pi)$ under the regularity condition. 
However, this condition is sometimes prone to be violated, particularly when \textit{Slater's condition} does not hold (i.e., there is no strictly feasible solution).
For example, if the original LP has a constraint $\x\ge\zeros_n$ and every column of $\P$ has opposite-sign entries, it is likely that only $\y = \zeros_k$ satisfies $\P\y\ge\zeros_n$ by equality, which is the unique optimal solution but not strictly feasible. 
In this case, the regularity condition is violated since all rows of $\P \in \R^{n \times k}$ are active at $\zeros_k$ and linearly dependent due to $n > k$.
To alleviate this issue, we apply the following projection for $j=1,\dots,k$ before computing the gradient in \eqref{eq:gradient}:
\begin{equation}\label{eq:projection}
  \P_{:,j} \gets \argmin_{\x\in\R^n}\Set{\norm{\x - \P_{:,j}}_2}{\A\x \le \b}, 
\end{equation} 
where $\P_{:,j}$ denotes the $j$th column of $\P$.
This minimally changes each column $\P_{:,j}$ to satisfy the original constraints. 
Consequently, any convex combination of $\P$'s columns is feasible for the original LP, increasing the chance that there exists a strictly feasible solution in $\Set{\y\in\R^k}{\A\P\y\le\b}$, although it is not guaranteed.
This improves the likelihood that the regularity condition is satisfied.

  Given $N$ training instances, $\pi_1,\dots,\pi_N$, we repeatedly update $\P$ as with SGD: for each $\pi_i$, we iterate to compute the gradient~\eqref{eq:gradient} and to update $\P$ with it. 
The projection \eqref{eq:projection} onto the feasible region of $\pi_i$ comes before computing the gradient for $\pi_i$.
We call this method SGA (stochastic gradient ascent).

\subsection{Final projection for feasibility}\label{subsec:impl}
The previous discussion suggests that making each column of $\P$ feasible for training LP instances can increase the likelihood that future LP instances projected by $\P$ will have strictly feasible solutions. 
Considering this, after obtaining a projection matrix $\P$ with either the PCA- or gradient-based method, we project each column of $\P$ onto the intersection of the feasible regions of training LP instances, which we call the \textit{final projection}. 
This can be done similarly to \eqref{eq:projection} replacing the constraints with $[\A_1;\dots;\A_N]\x \le [\b_1;\dots;\b_N]$. 
If $\A_1,\dots,\A_N$ are identical, we can do it more efficiently by replacing the constraints with $\A_1 \x \le \min\set{\b_1,\dots,\b_N}$, where the minimum is taken element-wise.
Note that although the final projection can be costly for large $N$, we need to do it only once at the end of learning $\P$.
This final projection never fails since $\zeros_n$ is always feasible as in \cref{assump:feasible_bounded}.

\newcommand{\noise}{\omega}
\section{Experiments}\label{section:experiment}

We experimentally evaluate the data-driven projection approach.
We used MacBook Air with Apple M2 chip, 24 GB of memory, and macOS Sonoma~14.1. 
We implemented algorithms in Python~3.9.7 with NumPy~1.23.2. 
We used Gurobi~10.0.1 \citep{Gurobi_Optimization_LLC2023-ae} for solving LPs and computing projection in \eqref{eq:projection}. 
We used the following three synthetic and five realistic datasets, each of which consists of $300$ LP instances ($200$ for training and $100$ for testing).
\Cref{tab:problem-size} summarizes LP sizes of the eight datasets.\footnote{While Gurobi can solve larger LPs, we used the moderate-size LPs as the learning methods could take much longer with the limited computational resources. We admit that larger instances might introduce new challenges. Nevertheless, the trends observed in our experiments offer informative insights for larger scenarios as well.\label{footnote:problem-size}}

\begin{table}[t]
  \centering
  \begin{small}
  \caption{Sizes of inequality-form LPs, where $m$ ($n$) represents the number of constraints (variables).}
  \label[table]{tab:problem-size}
  \begin{tabular}{lcccccccc}
  \toprule
   & Packing & MaxFlow & MinCostFlow & GROW7 & ISRAEL & SC205 & SCAGR25 & STAIR\\
  \midrule
  $m$ & 50 & 1000 & 1000 & 581 & 316 & 317 & 671 & 696 \\
  $n$ & 500 & 500 & 500  & 301 & 142 & 203 & 500 & 467 \\
  \bottomrule
  \end{tabular}
  \end{small}
\end{table}

\begin{figure*}[htb]
  \centering
  \includegraphics[width=.95\linewidth]{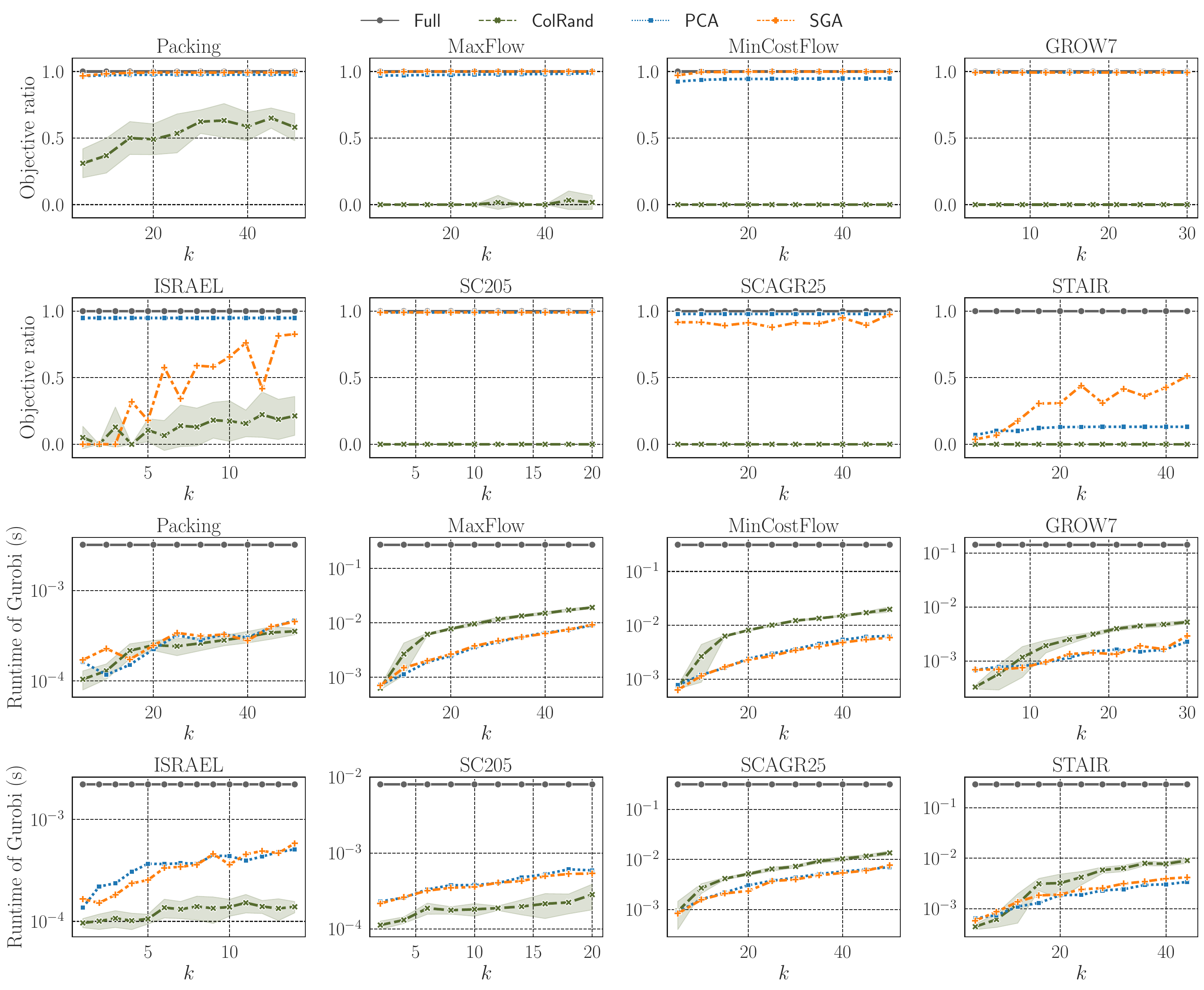}
  \caption{Plots of objective ratios (upper) and Gurobi's running times (lower, semi-log) for \Full, \ColRand, \PCA, and \SGA averaged over 100 test instances.
  The error band of \ColRand indicates the standard deviation over $10$ independent trials. 
  The results of \Full are shown for every $k$ for reference, although it always solves $n$-dimensional LPs and hence is independent of $k$. 
  }
  \label{fig:results}
\end{figure*}

\textbf{Synthetic datasets.}
We consider three types of LPs representing packing, maximum flow, and minimum-cost flow problems, denoted by Packing, MaxFlow, and MinCostFlow, respectively.
A packing problem is an LP with non-negative parameters $\c$, $\A$, and $\b$. 
We created a base instance by drawing their entries from the uniform distribution on $[0, 1]$ and multiplying $\b$ by $n$. 
We then obtained 300 random instances by multiplying all input parameters by $1 + \noise$, where $\noise$ was drawn from the uniform distribution on $[0, 0.1]$. 
To generate MaxFlow and MinCostFlow LPs, we first randomly created a directed graph with $50$ vertices and $500$ arcs and fixed source and sink vertices, denoted by $s$ and $t$, respectively.
We confirmed there was an arc from $s$ to $t$ to ensure feasibility. 
We set base arc capacities to $1$, which we perturbed by multiplying $1 + \noise$ with $\noise$ drawn from the uniform distribution on $[0, 0.1]$, thus obtaining 300 MaxFlow instances.
For MinCostFlow, we set supply at $s$ and demand at $t$ to $1$. 
We set base arc costs to $1$ for all arcs but $(s, t)$, whose cost was fixed to be large enough, and perturbed them similarly using $1 + \noise$ to obtain 300 MinCostFlow instances. 
We transformed MaxFlow and MinCostFlow instances into equivalent inequality-form LPs with a method given in \cref{app:remove-eq}, which requires a (trivially) feasible solution $\x_0$. 
For MaxFlow, we used $\x_0 = \zeros$ (i.e., no flow) as a trivially feasible solution. 
For MinCostFlow, we let $\x_0$ be all zeros but a single $1$ at the entry corresponding to $(s, t)$, which is a trivially feasible (but costly) solution.

\textbf{Realistic datasets.}
We used five LPs in Netlib \citep{Browne1995-wd}, GROW7, ISRAEL, SC205, SCAGR25, and STAIR. 
For each, we generated datasets of 300 random instances.
To create realistic datasets, we made them contain 2\% of outliers as follows. 
For normal 98\% data points, we perturbed coefficients of objective functions by multiplying $1+0.1\noise$, where $\noise$ was drawn from the normal distribution; 
for 2\% outliers, we perturbed them by multiplying $1+\noise$, i.e., 10 times larger noises. 
Except for ISRAEL, the LPs have equality constraints. 
We transformed them into inequality-form LPs as described in \cref{app:remove-eq}, using $\x_0$ found by the initialization procedure of Gurobi's interior-point method.\footnote{We expect this $\x_0$ is (close to) the \textit{analytic center}, although we could not verify Gurobi's internal processes.}

\textbf{Methods.}
We compared four methods, named \Full, \ColRand, \PCA, and \SGA.
The first two are baseline methods, while the latter two are our data-driven projection methods.
Note that all four methods solved LPs with Gurobi, the state-of-the-art commercial solver. 
The only difference among them lies in how to reduce the dimensionality of LPs, as detailed below.

\textbf{\Full:} a baseline method that returns original $n$-dimensional LPs without reducing the dimensionality. 

\textbf{\ColRand:} a column-randomized method based on the work by \citet{Akchen2024-gg}, which reduces the dimensionality by selecting $k$ out of $n$ variables randomly and fixing the others to zeros. 

\textbf{\PCA:} the PCA-based method that reduces the dimensionality with a projection matrix $\P$ learned as in \cref{subsec:pca}, followed by the final projection described in \cref{subsec:impl}.

\textbf{\SGA:} the gradient-based method that learns $\P$ as described in \cref{subsec:gradient}, followed by the final projection as with \PCA.  
We initialized $\P$ with that obtained by \PCA and conducted a single epoch of training, setting the learning rate to $0.01$.\footnote{While we also tried \SGA with the random initialization, we found that the \PCA-initialization worked better. We present the results with the random initialization in \cref{fig:results_rand_init} in the appendix for completeness.\label{footnote:rand-init}}

For \ColRand, \PCA, and \SGA, we used increasing values of the reduced dimensionality, $k = \floor*{\frac{n}{100}}, 2\floor*{\frac{n}{100}}, \dots$, until it reached the maximum value no more than $\floor*{\frac{n}{10}}$, i.e., up to 10\% of the original dimensionality. 
\PCA and \SGA learned projection matrices $\P$ from $N=200$ training instances, which were then used to reduce the dimensionality of 100 test instances.
For \ColRand, we tried $10$ independent choices of $k$ variables and recorded the average and standard deviation.

\textbf{Results.}
\Cref{fig:results} shows how the solution quality and running time of Gurobi differ among the four methods, where ``objective ratio'' means the objective value divided by the optimal value computed by \Full. 
For all datasets except STAIR, \PCA and/or \SGA with the largest $k$ achieved about 95\% to 99\% objective ratios, while being about 4 to 70 times faster than \Full. 
Regarding STAIR, \PCA and \SGA attained 13.1\% and 51.2\% objective ratios, respectively. 
By stark contrast, \ColRand resulted in objective ratios close to zero in most cases except for Packing and ISRAEL. 
The results suggest that given informative training datasets, data-driven projection methods can lead to significantly better solutions than the random projection method. 
Regarding running times, there were differences between \PCA/\SGA and \ColRand, which were probably caused by the numerical property of Gurobi. 
Nonetheless, all of the three were substantially faster than \Full.
In summary, the data-driven projection methods achieve high solution quality while greatly reducing the time for solving LPs.

Comparing \PCA and \SGA, \SGA achieved better objectives than \PCA in Packing, MaxFlow, MinCostFlow, and STAIR, while performing similarly in GROW7 and SC205. 
In ISRAEL and SCAGR25, \SGA was worse than \PCA, but this is not surprising since optimizing $u(\P, \pi)$ is a non-convex problem. 
The results suggest that no method could be universally best. 
Fortunately, the generalization bound~\eqref{eq:generalization-bound} justifies selecting a learning method based on empirical performance. 
Specifically, if we adopt a learning method that produces $\P$ with the best empirical performance on $N$ instances at hand, its expected performance on future instances is likely to stay within the range of $\pm \varepsilon$ of the empirical one, where $\varepsilon\lesssim H\sqrt{{\pdim(\Ucal)}/{N}} \lesssim Hk\sqrt{{n}/{N}}$ since $\pdim(\Ucal) = \tilde{\Ord}(nk^2)$ due to \cref{theorem:pdim-upper-bound}.

\section{Conclusion}
We have studied the data-driven projection approach to LPs.
We have established a generalization bound by proving an $\tilde{\Ord}(nk^2)$ upper bound on the pseudo-dimension and complemented it by an $\Omega(nk)$ lower bound. 
We have also proposed PCA- and gradient-based learning methods and experimentally evaluated them. 
Our theoretical and empirical findings lay the groundwork for the further development of the data-driven approach to LPs and contribute to the broader trend of AI/ML for optimization~\citep{Van_Hentenryck2024-qz}.

\section{Limitations and discussions}\label{sec:limitations}
Our work is limited to the statistical learning setting with assumptions on LP instances (see \cref{subsec:data-driven-projection} and \cref{assump:feasible_bounded}).
In particular, the current approach cannot deal with equality constraints varying across instances since they usually make LP instances have no common feasible solution.
Despite the narrowed applicability, we believe our setting is a reasonable starting point for developing the data-driven projection approach to LPs, as is also discussed in \cref{rem:validity} and the paragraph following \cref{assump:feasible_bounded}.
Overcoming these limitations will require more involved methods, such as training neural networks to extract meaningful low-dimensional subspaces from non-i.i.d.\ messy LP instances.

Our learning methods are not efficient, and applying them to huge LPs in practice might be challenging. 
Similar challenges are common in most data-driven algorithm research, as discussed in \cref{rem:learning-time}, and we believe our conceptually simple learning methods are helpful for future research.
In numerical linear algebra, \citet{Indyk2021-yn} found that few-shot learning methods can efficiently learn sketching matrices. 
We expect similar ideas will be useful for learning projection matrices for LPs efficiently.
For the same reason, our experiments are limited to moderate-size LPs, as mentioned in \cref{footnote:problem-size}.
Nevertheless, the results sufficiently serve as a proof of concept of the data-driven projection approach.

There also exist general limitations of the projection-based approach \citep{Vu2018-ty,Poirion2023-bm,Akchen2024-gg}.
First, it does not consider solver-specific aspects, including numerical stability and sparsity, as discussed in \cref{rem:solver-specific}. 
Second, the projection-based approach has a limited impact on the theoretical time complexity.
The theoretical time complexity of the projection-based approach is dominated by two factors:
multiplying $\P$ to reduce the dimensionality and solving the projected LP.
Recent theoretical studies have revealed that solving an LP takes asymptotically the same computation time as matrix multiplication~\citep{Cohen2021-qt,Jiang2021-ih}, suggesting projections may not contribute to improving the total theoretical time complexity.
Nevertheless, the projection-based approach leads to dramatic speedups in practice, as in \cref{fig:results}.
Moreover, it can be even faster beyond the theoretical implications when GPUs are available.
It is noteworthy that the projection-based approach largely benefits from GPUs, as matrix multiplication can be highly parallelized.
An exciting future direction is to combine recent GPU-implemented LP solvers \citep{Applegate2021-ab,Lu2023-ks,Lu2023-zi} with projections, which will have vast potential for solving huge LPs efficiently.

\subsection*{Acknowledgements}
This work was supported by JST ERATO Grant Number JPMJER1903 and JSPS KAKENHI Grant Number JP22K17853.

\renewcommand{\c}{\origc}
\newrefcontext[sorting=nyt]
\printbibliography[heading=bibintoc]

@INPROCEEDINGS{Garber2021-tg,
  title     = "{Frank--Wolfe} with a Nearest Extreme Point Oracle",
  booktitle = "Proceedings of the 34th Conference on Learning Theory (COLT 2021)",
  author    = "Garber, Dan and Wolf, Noam",
  editor    = "Belkin, Mikhail and Kpotufe, Samory",
  publisher = "PMLR",
  volume    =  134,
  pages     = "2103--2132",
  series    = "Proceedings of Machine Learning Research",
  year      =  2021
}

@INPROCEEDINGS{Lacoste-Julien2015-ir,
  title     = "On the Global Linear Convergence of {F}rank--{W}olfe
               Optimization Variants",
  booktitle = "Advances in Neural Information Processing Systems (NIPS 2015)",
  author    = "Lacoste-Julien, Simon and Jaggi, Martin",
  publisher = "Curran Associates, Inc.",
  volume    =  28,
  pages     = "496--504",
  year      =  2015
}

@ARTICLE{Van_Hentenryck2024-qz,
  title   = "{AI4OPT}: {AI} Institute for Advances in Optimization",
  author  = "Van Hentenryck, Pascal and Dalmeijer, Kevin",
  journal = "AI Magazine",
  volume  =  45,
  number  =  1,
  pages   = "42--47",
  year    =  2024
}

@ARTICLE{Lu2023-ks,
  title         = "{cuPDLP.jl}: A {GPU} Implementation of Restarted Primal-Dual
                   Hybrid Gradient for Linear Programming in {Julia}",
  author        = "Lu, Haihao and Yang, Jinwen",
  month         =  nov,
  year          =  2023,
  archivePrefix = "arXiv",
  primaryClass  = "math.OC",
  eprint        = "2311.12180",
  journal       = "arXiv:2311.12180"
}

@INPROCEEDINGS{Applegate2021-ab,
  title     = "Practical Large-Scale Linear Programming using Primal-Dual
               Hybrid Gradient",
  booktitle = "Advances in Neural Information Processing Systems",
  author    = "Applegate, David and Diaz, Mateo and Hinder, Oliver and Lu,
               Haihao and Lubin, Miles and O'Donoghue, Brendan and Schudy, Warren",
  editor    = "Ranzato, M and Beygelzimer, A and Dauphin, Y and Liang, P S and
               Vaughan, J Wortman",
  publisher = "Curran Associates, Inc.",
  volume    =  34,
  pages     = "20243--20257",
  year      =  2021
}

@ARTICLE{Lu2023-zi,
  title         = "{cuPDLP-C}: A Strengthened Implementation of {cuPDLP} for Linear
                   Programming by {C} language",
  author        = "Lu, Haihao and Yang, Jinwen and Hu, Haodong and Huangfu, Qi
                   and Liu, Jinsong and Liu, Tianhao and Ye, Yinyu and Zhang,
                   Chuwen and Ge, Dongdong",
  month         =  dec,
  year          =  2023,
  archivePrefix = "arXiv",
  primaryClass  = "math.OC",
  eprint        = "2312.14832",
  journal       = "arXiv:2312.14832"
}

@ARTICLE{Browne1995-wd,
  title   = "The {Netlib} mathematical software repository",
  author  = "Browne, S and Dongarra, J and Grosse, E and Rowan, T",
  journal = "D-lib Magazine",
  year    =  1995,
  url     = "https://www.netlib.org/lp/data/",
  note    = "The dataset used in this study is freely available at \url{https://www.netlib.org/lp/data/}."
}

@BOOK{Pollard1984-zp,
  title     = "Convergence of Stochastic Processes",
  author    = "Pollard, David",
  publisher = "Springer",
  edition   = "1st",
  year      =  1984
}

@ARTICLE{Gupta2017-ng,
  title     = "A {PAC} Approach to Application-Specific Algorithm Selection",
  author    = "Gupta, Rishi and Roughgarden, Tim",
  journal   = "SIAM Journal on Computing",
  publisher = "SIAM",
  volume    =  46,
  number    =  3,
  pages     = "992--1017",
  year      =  2017
}

@ARTICLE{Amos2023-kw,
  title   = "Tutorial on amortized optimization",
  author  = "Amos, Brandon",
  journal = "Foundations and Trends\textregistered{} in Machine Learning",
  volume  =  16,
  number  =  5,
  pages   = "1935--8237",
  year    =  2023
}

@BOOK{Anthony2009-mm,
  title     = "Neural Network Learning: Theoretical Foundations",
  author    = "Anthony, Martin and Bartlett, Peter L",
  publisher = "Cambridge University Press",
  year      =  2009
}

@INPROCEEDINGS{Fan2023-gr,
  title     = "Smart Initial Basis Selection for Linear Programs",
  booktitle = "Proceedings of the 40th International Conference on Machine
               Learning (ICML 2023)",
  author    = "Fan, Zhenan and Wang, Xinglu and Yakovenko, Oleksandr and Sivas,
               Abdullah Ali and Ren, Owen and Zhang, Yong and Zhou, Zirui",
  editor    = "Krause, Andreas and Brunskill, Emma and Cho, Kyunghyun and
               Engelhardt, Barbara and Sabato, Sivan and Scarlett, Jonathan",
  publisher = "PMLR",
  volume    =  202,
  pages     = "9650--9664",
  series    = "Proceedings of Machine Learning Research",
  year      =  2023
}

@INPROCEEDINGS{Berthet2020-wm,
  title     = "Learning with Differentiable Pertubed Optimizers",
  booktitle = "Advances in Neural Information Processing Systems (NeurIPS 2020)",
  author    = "Berthet, Quentin and Blondel, Mathieu and Teboul, Olivier and
               Cuturi, Marco and Vert, Jean-Philippe and Bach, Francis",
  editor    = "Larochelle, H and Ranzato, M and Hadsell, R and Balcan, M F and
               Lin, H",
  publisher = "Curran Associates, Inc.",
  volume    =  33,
  pages     = "9508--9519",
  year      =  2020
}

@INPROCEEDINGS{Agrawal2019-sv,
  title     = "Differentiable Convex Optimization Layers",
  booktitle = "Advances in Neural Information Processing Systems (NeurIPS 2019)",
  author    = "Agrawal, Akshay and Amos, Brandon and Barratt, Shane and Boyd,
               Stephen and Diamond, Steven and Kolter, J Zico",
  publisher = "Curran Associates, Inc.",
  volume    =  32,
  pages     = "9562--9574",
  year      =  2019
}

@ARTICLE{Elmachtoub2022-og,
  title   = "Smart ``Predict, then Optimize''",
  author  = "Elmachtoub, Adam N and Grigas, Paul",
  journal = "Management Science",
  volume  =  68,
  number  =  1,
  pages   = "9--26",
  year    =  2022
}

@INPROCEEDINGS{Amos2017-sf,
  title     = "{O}pt{N}et: Differentiable Optimization as a Layer in Neural
               Networks",
  booktitle = "Proceedings of the 34th International Conference on Machine
               Learning (ICML 2017)",
  author    = "Amos, Brandon and Kolter, J Zico",
  publisher = "PMLR",
  volume    =  70,
  pages     = "136--145",
  year      =  2017
}

@INPROCEEDINGS{Wilder2019-yw,
  title     = "Melding the Data-Decisions Pipeline: Decision-Focused Learning
               for Combinatorial Optimization",
  booktitle = "Proceedings of the 33rd AAAI Conference on Artificial
               Intelligence (AAAI 2019)",
  author    = "Wilder, Bryan and Dilkina, Bistra and Tambe, Milind",
  volume    =  33,
  pages     = "1658--1665",
  month     =  jul,
  year      =  2019
}

@INPROCEEDINGS{Sakaue2022-bb,
  title     = "Sample Complexity of Learning Heuristic Functions for
               Greedy-Best-First and {A*} Search",
  booktitle = "Advances in Neural Information Processing Systems (NeurIPS 2022)",
  author    = "Sakaue, Shinsaku and Oki, Taihei",
  editor    = "Koyejo, S and Mohamed, S and Agarwal, A and Belgrave, D and Cho,
               K and Oh, A",
  publisher = "Curran Associates, Inc.",
  volume    =  35,
  pages     = "2889--2901",
  year      =  2022
}

@INPROCEEDINGS{Sakaue2023-ta,
  title     = "Improved Generalization Bound and Learning of Sparsity Patterns
               for Data-Driven Low-Rank Approximation",
  booktitle = "Proceedings of the 26th International Conference on Artificial
               Intelligence and Statistics (AISTATS 2023)",
  author    = "Sakaue, Shinsaku and Oki, Taihei",
  editor    = "Ruiz, Francisco and Dy, Jennifer and van de Meent, Jan-Willem",
  publisher = "PMLR",
  volume    =  206,
  pages     = "1--10",
  series    = "Proceedings of Machine Learning Research",
  year      =  2023
}

@INPROCEEDINGS{Li2023-sp,
  title     = "Learning the Positions in {CountSketch}",
  booktitle = "International Conference on Learning Representations (ICLR 2023)",
  author    = "Li, Yi and Lin, Honghao and Liu, Simin and Vakilian, Ali and
               Woodruff, David",
  year      =  2023
}

@INPROCEEDINGS{Indyk2019-cn,
  title     = "Learning-Based Low-Rank Approximations",
  booktitle = "Advances in Neural Information Processing Systems (NeurIPS 2019)",
  author    = "Indyk, Piotr and Vakilian, Ali and Yuan, Yang",
  publisher = "Curran Associates, Inc.",
  volume    =  32,
  pages     = "7402--7412",
  year      =  2019
}

@INPROCEEDINGS{Indyk2021-yn,
  title     = "Few-Shot Data-Driven Algorithms for Low Rank Approximation",
  booktitle = "Advances in Neural Information Processing Systems (NeurIPS 2021)",
  author    = "Indyk, Piotr and Wagner, Tal and Woodruff, David",
  publisher = "Curran Associates, Inc.",
  volume    =  34,
  pages     = "10678--10690",
  year      =  2021
}

@ARTICLE{Woodruff2014-ya,
  title   = "Sketching as a Tool for Numerical Linear Algebra",
  author  = "Woodruff, David P",
  journal = "Foundations and Trends\textregistered{} in Theoretical Computer
             Science",
  volume  =  10,
  number  = "1--2",
  pages   = "1--157",
  year    =  2014
}

@ARTICLE{Vu2018-ty,
  title   = "Random Projections for Linear Programming",
  author  = "Vu, Ky and Poirion, Pierre-Louis and Liberti, Leo",
  journal = "Mathematics of Operations Research",
  volume  =  43,
  number  =  4,
  pages   = "1051--1071",
  year    =  2018
}

@ARTICLE{Poirion2023-bm,
  title   = "Random projections of linear and semidefinite problems with linear
             inequalities",
  author  = "Poirion, Pierre-Louis and Louren{\c c}o, Bruno F and Takeda, Akiko",
  journal = "Linear Algebra and its Applications",
  volume  =  664,
  pages   = "24--60",
  year    =  2023
}

@ARTICLE{Huangfu2018-ng,
  title   = "Parallelizing the dual revised simplex method",
  author  = "Huangfu, Q and Hall, J A J",
  journal = "Mathematical Programming Computation",
  volume  =  10,
  number  =  1,
  pages   = "119--142",
  month   =  mar,
  year    =  2018
}

@INPROCEEDINGS{Berthold2021-wk,
  title     = "Learning To Scale Mixed-Integer Programs",
  booktitle = "Proceedings of the 35th AAAI Conference on Artificial
               Intelligence (AAAI 2021)",
  author    = "Berthold, Timo and Hendel, Gregor",
  volume    =  35,
  pages     = "3661--3668",
  year      =  2021
}

@ARTICLE{Akchen2024-gg,
  title   = "Column-Randomized Linear Programs: Performance Guarantees and
             Applications",
  author  = "Akchen, Yi-Chun and Mi{\v s}i{\'c}, Velibor V.",
  journal = "Operations Research",
  volume  =  0,
  number  =  0,
  year    =  2024
}

@INPROCEEDINGS{Sun2023-yc,
  title     = "Learning to Generate Columns with Application to Vertex Coloring",
  booktitle = "International Conference on Learning Representations (ICLR 2023)",
  author    = "Sun, Yuan and Ernst, Andreas T and Li, Xiaodong and Weiner, Jake",
  year      =  2023
}

@ARTICLE{El_Balghiti2023-xe,
  title   = "Generalization Bounds in the Predict-Then-Optimize Framework",
  author  = "El Balghiti, Othman and Elmachtoub, Adam N. and Grigas, Paul and Tewari, Ambuj",
  journal = "Mathematics of Operations Research",
  publisher = "INFORMS",
  volume  =  48,
  number  =  4,
  pages   = "2043--2065",
  year    =  2023
}

@INPROCEEDINGS{Bartlett2022-mu,
  title     = "Generalization Bounds for Data-Driven Numerical Linear Algebra",
  booktitle = "Proceedings of the 35th Conference on Learning Theory (COLT 2022)",
  author    = "Bartlett, Peter L and Indyk, Piotr and Wagner, Tal",
  publisher = "PMLR",
  volume    =  178,
  pages     = "2013--2040",
  year      =  2022
}

@INPROCEEDINGS{Balcan2022-yp,
  title     = "Structural Analysis of Branch-and-Cut and the Learnability of
               {Gomory} Mixed Integer Cuts",
  booktitle = "Advances in Neural Information Processing Systems (NeurIPS 2022)",
  author    = "Balcan, Maria-Florina and Prasad, Siddharth and Sandholm, Tuomas
               and Vitercik, Ellen",
  editor    = "Koyejo, S and Mohamed, S and Agarwal, A and Belgrave, D and Cho,
               K and Oh, A",
  publisher = "Curran Associates, Inc.",
  volume    =  35,
  pages     = "33890--33903",
  year      =  2022
}

@ARTICLE{Chowdhury2022-xy,
  title    = "Faster Randomized Interior Point Methods for Tall/Wide Linear
              Programs",
  author   = "Chowdhury, Agniva and Dexter, Gregory and London, Palma and
              Avron, Haim and Drineas, Petros",
  journal  = "Journal of Machine Learning Research",
  volume   =  23,
  number   =  336,
  pages    = "1--48",
  year     =  2022
}

@INPROCEEDINGS{Balcan2022-em,
  title     = "Improved sample complexity bounds for branch-and-cut",
  booktitle = "Proceedings of the 28th International Conference on Principles
               and Practice of Constraint Programming (CP 2022)",
  author    = "Balcan, Maria-Florina and Prasad, Siddharth and Sandholm, Tuomas
               and Vitercik, Ellen",
  publisher = "Schloss Dagstuhl -- Leibniz-Zentrum f{\"u}r Informatik",
  year      =  2022
}

@INPROCEEDINGS{Balcan2021-jv,
  title     = "How much data is sufficient to learn high-performing algorithms?
               {Generalization} guarantees for data-driven algorithm design",
  booktitle = "Proceedings of the 53rd Annual ACM SIGACT Symposium on Theory of
               Computing (STOC 2021)",
  author    = "Balcan, Maria-Florina and DeBlasio, Dan and Dick, Travis and
               Kingsford, Carl and Sandholm, Tuomas and Vitercik, Ellen",
  publisher = "ACM",
  pages     = "919--932",
  year      =  2021
}

@INCOLLECTION{Balcan2021-fy,
  title     = "Data-Driven Algorithm Design",
  booktitle = "{Beyond the Worst-Case Analysis of Algorithms}",
  author    = "Balcan, Maria-Florina",
  publisher = "Cambridge University Press",
  pages     = "626--645",
  year      =  2021
}

@INPROCEEDINGS{Tan2020-lp,
  title     = "Learning Linear Programs from Optimal Decisions",
  booktitle = "Advances in Neural Information Processing Systems (NeurIPS 2020)",
  author    = "Tan, Yingcong and Terekhov, Daria and Delong, Andrew",
  editor    = "Larochelle, H and Ranzato, M and Hadsell, R and Balcan, M F and
               Lin, H",
  publisher = "Curran Associates, Inc.",
  volume    =  33,
  pages     = "19738--19749",
  year      =  2020
}

@INPROCEEDINGS{Wang2020-wt,
  title     = "Automatically Learning Compact Quality-aware Surrogates for
               Optimization Problems",
  booktitle = "Advances in Neural Information Processing Systems (NeurIPS 2020)",
  author    = "Wang, Kai and Wilder, Bryan and Perrault, Andrew and Tambe,
               Milind",
  editor    = "Larochelle, H and Ranzato, M and Hadsell, R and Balcan, M F and
               Lin, H",
  publisher = "Curran Associates, Inc.",
  volume    =  33,
  pages     = "9586--9596",
  year      =  2020
}

@BOOK{Korte2012-gy,
  title     = "Cobinatorial Optimization: Theory and Algorithams",
  author    = "Korte, Bernard and Vygen, Jens",
  publisher = "Springer",
  edition   = "5th",
  year      =  2012
}

@ARTICLE{Goldberg1995-af,
  title   = "Bounding the {Vapnik--Chervonenkis} dimension of concept classes
             parameterized by real numbers",
  author  = "Goldberg, Paul W and Jerrum, Mark R",
  journal = "Machine Learning",
  volume  =  18,
  number  =  2,
  pages   = "131--148",
  year    =  1995
}

@ARTICLE{Warren1968-hp,
  title     = "Lower Bounds for Approximation by Nonlinear Manifolds",
  author    = "Warren, Hugh E",
  journal   = "Transactions of the American Mathematical Society",
  publisher = "American Mathematical Society",
  volume    =  133,
  number    =  1,
  pages     = "167--178",
  year      =  1968
}

@ARTICLE{Cohen2021-qt,
  title     = "Solving Linear Programs in the Current Matrix Multiplication
               Time",
  author    = "Cohen, Michael B and Lee, Yin Tat and Song, Zhao",
  journal   = "Journal of the ACM",
  publisher = "ACM",
  volume    =  68,
  number    =  1,
  pages     = "1--39",
  month     =  jan,
  year      =  2021,
  address   = "New York, NY, USA"
}

@INPROCEEDINGS{Jiang2021-ih,
  title     = "A faster algorithm for solving general {LPs}",
  booktitle = "Proceedings of the 53rd Annual ACM SIGACT Symposium on Theory of
               Computing (STOC 2021)",
  author    = "Jiang, Shunhua and Song, Zhao and Weinstein, Omri and Zhang,
               Hengjie",
  publisher = "ACM",
  pages     = "823--832",
  series    = "STOC 2021",
  month     =  jun,
  year      =  2021,
  address   = "New York, NY, USA",
  location  = "Virtual, Italy"
}

@MISC{Gurobi_Optimization_LLC2023-ae,
  title        = "{Gurobi Optimizer Reference Manual}",
  author       = "{Gurobi Optimization, LLC}",
  year         =  2023,
  howpublished = "\url{https://www.gurobi.com}",
  note = "Used under academic license."
}

@BOOK{Gass1985-dz,
  title     = "Linear Programming: Methods and Applications",
  author    = "Gass, Saul I.",
  publisher = "McGraw-Hill",
  year      =  1985,
  language  = "en"
}

@BOOK{Eiselt2007-an,
  title     = "Linear Programming and its Applications",
  author    = "Eiselt, H. A. and Sandblom, C.-L.",
  publisher = "Springer",
  month     =  aug,
  year      =  2007,
  language  = "en"
}

@ARTICLE{Balcan2023-dn,
  title     = "Learning to branch: Generalization guarantees and limits of
               data-independent discretization",
  author    = "Balcan, Maria-Florina and Dick, Travis and Sandholm, Tuomas and
               Vitercik, Ellen",
  journal   = "Journal of the ACM",
  publisher = "ACM",
  month     =  dec,
  year      =  2023,
  address   = "New York, NY, USA"
}

@INPROCEEDINGS{Meng2021-yb,
  title     = "Differentiable Optimization of Generalized Nondecomposable
               Functions using Linear Programs",
  booktitle = "Advances in Neural Information Processing Systems (NeurIPS 2021)",
  author    = "Meng, Zihang and Mukherjee, Lopamudra and Wu, Yichao and Singh,
               Vikas and Ravi, Sathya",
  editor    = "Ranzato, M and Beygelzimer, A and Dauphin, Y and Liang, P S and
               Vaughan, J Wortman",
  publisher = "Curran Associates, Inc.",
  volume    =  34,
  pages     = "29129--29141",
  year      =  2021
}
\renewcommand{\c}{\bm{c}}

\clearpage
\appendix
\onecolumn

{\noindent \LARGE \bf Appendix}

\section{Additional related work on learning through optimization}
\label{app:related-work}
Many researchers have addressed learning tasks whose input--output pipelines involve optimization steps~\citep{Amos2017-sf,Wilder2019-yw,Agrawal2019-sv,Berthet2020-wm,Tan2020-lp,Meng2021-yb,Amos2023-kw,Elmachtoub2022-og,Wang2020-wt,El_Balghiti2023-xe}.
While most of them seek to develop practical learning methods, several have studied generalization guarantees. 
\citet{Wang2020-wt} have studied a so-called \textit{decision-focused learning} method with \textit{reparametrization}, which is technically the same as projection. 
Their theoretical result focuses on learning of models that generate objective functions, assuming reparametrization matrices to be fixed.
By contrast, we obtain a generalization bound for learning projection matrices. 
\citet{El_Balghiti2023-xe} have studied generalization bounds in the so-called \textit{smart predict-then-optimize} setting. 
While they focus on learning of models that generate coefficients of objectives from contextual information as with \citet{Wang2020-wt}, our interest is in learning projection matrices, which affect both objectives and constraints. 
On the practical side, the line of work provides useful techniques for differentiating outcomes of optimization with respect to input parameters. 
Our gradient-based method for learning projection matrices is partly based on the result by \citet{Tan2020-lp}.

\section{Proof of the lower bound on $\pdim(\Ucal)$}\label{app:proof-lower-bound}
We establish the lower bound on $\pdim(\Ucal)$ in \cref{theorem:pdim-lower-bound} by constructing a set of $(n-2k)k$ LP instances that $\Ucal$ can shatter.
  The instances are written as $\pi_{\itt,\jtt} = (\c_\itt, \A, \b_\jtt) \in \R^n\times\R^{2k\times n}\times\R^{2k}$ for $\itt = 1,\dots,n-2k$ and $\jtt = 1,\dots,k$, where 
  \begin{align}
    \c_\itt = \begin{bmatrix}
      \e_\itt \\ \zeros_{2k}
    \end{bmatrix},
    &&
    \A = [\zeros_{2k, n-2k}, \I_{2k}],
    &&
    \text{and}
    &&
    \b_\jtt = \begin{bmatrix}
      \e_\jtt \\ \zeros_{k}
    \end{bmatrix}.  
  \end{align}
Here, $\e_r$ and $\e_s$ are the $r$th and $s$th standard basis vectors of $\R^{n-2k}$ and $\R^{k}$, respectively, and $\zeros_{a, b}$ is the $a\times b$ all zeros.
We consider a projection matrix of the form 
\[
  \P = 
  \begin{bmatrix}
    \Q \\ \I_k \\ -\I_k
  \end{bmatrix}, 
\]
where $\Q \in \set{0,1}^{(n-2k) \times k}$ is a binary matrix that we will use as tunable parameters to shatter the set of $(n-2k)k$ instances. 
Let $y_j$ denote the $j$th entry of the variable vector $\y\in\R^k$. 
Since we have 
\[
  \A\P = \begin{bmatrix}\I_k \\ -\I_k \end{bmatrix}
\]
the constraints, $\A\P\y \le \b_\jtt$, imply $y_j = 0$ for $j=1,\dots,k$ with $j\neq \jtt$ and $y_s \in [0, 1]$. 
Let $\y$ be such a feasible solution. 
Then, the objective value is written as $\c_\itt^\top \P\y = \e_\itt^\top\Q\y = Q_{r,s}y_s$, where $Q_{r,s}$ is the $(r,s)$ entry of $\Q \in \set{0,1}^{(n-2k) \times k}$. 
Since $Q_{r,s} \in \set{0,1}$ and $y_s \in [0, 1]$, we have $\max\Set{\c_\itt^\top \P \y}{\A\P\y \le \b_\jtt} = Q_{r,s}$. 
Thus, the set of those $(n-2k)k$ instances can be shattered by setting all threshold values to $1/2$ and appropriately choosing each entry of $\Q \in \set{0,1}^{(n-2k) \times k}$. 
In other words, all the $2^{(n-2k)k}$ outcomes of 
$\Set{u(\P, \pi_{\itt,\jtt}) = Q_{r,s} \ge 1/2}{\itt=1,\dots,n-2k,\, \jtt=1,\dots,k}$ 
can realize by changing $\P \in \R^{n\times k}$ (or $\Q \in \set{0,1}^{(n-2k) \times k}$).
Thus, we obtain an $\Omega(nk)$ lower bound on the pseudo-dimension of $\Ucal = \Set*{u(\P, \cdot):\Pi\to\R}{\P \in \R^{n\times k}}$.

\section{How to remove equality constraints}\label{app:remove-eq}
Suppose that we are given an LP of the form 
\begin{equation}
  \mathop{\text{maximize}}_{\z\in\R^n} \quad \w^\top \z
  \qquad  
  \mathop{\text{subject to}}\quad \A_{\text{ineq}}\z \le \b_{\text{ineq}},\ \A_{\text{eq}}\z = \b_{\text{eq}}, 
\end{equation}
which has both inequality and equality constraints. 
Below, assuming that a (trivially) feasible solution $\x_0$ is available (i.e., $\A_{\text{ineq}}\x_0 \le \b_{\text{ineq}}$ and $\A_{\text{eq}}\x_0 = \b_{\text{eq}}$), we transform the LP into an equivalent inequality form. 
First, we replace the variable vector $\z$ with $\z' + \x_0$, obtaining an equivalent LP of the form 
\begin{equation}
  \mathop{\text{maximize}}_{\z'\in\R^n} \quad \w^\top (\z' + \x_0)
  \qquad  
  \mathop{\text{subject to}}\quad \A_{\text{ineq}}\z' \le \b_{\text{ineq}} - \A_{\text{ineq}}\x_0,\ \A_{\text{eq}}\z' = \zeros. 
\end{equation}
The equality constraints, $\A_{\text{eq}}\z' = \zeros$, mean that $\z'$ must be in the null space of $\A_{\text{eq}}$. 
Therefore, $\z'$ is always represented as $\z' = (\I - \A_{\text{eq}}^\dagger \A_{\text{eq}})\x$ with some $\x \in \R^n$, where $\A_{\text{eq}}^\dagger$ is the pseudo-inverse of $\A_{\text{eq}}$ and $\I - \A_{\text{eq}}^\dagger \A_{\text{eq}}$ is the orthogonal projection matrix onto the null space of $\A_{\text{eq}}$.\footnote{We can also represent $\z'$ by a linear combination of a basis of the null space of $\A_{\text{eq}}$, which we can compute via SVD. However, we found that this representation was numerically unstable in our experiments.} 
Substituting $\z' = (\I - \A_{\text{eq}}^\dagger \A_{\text{eq}})\x$ into the above LP allows us to remove the equality constraints since they are automatically satisfied for every $\x$. 
After all, by transforming the variable vector as $\z = (\I - \A_{\text{eq}}^\dagger \A_{\text{eq}})\x + \x_0$, we can obtain an equivalent LP of the form 
\begin{equation}
  \mathop{\text{maximize}}_{\x\in\R^n} \quad \w^\top (\I - \A_{\text{eq}}^\dagger \A_{\text{eq}})\x
  \qquad  
  \mathop{\text{subject to}}\quad \A_{\text{ineq}}(\I - \A_{\text{eq}}^\dagger \A_{\text{eq}})\x \le \b_{\text{ineq}} - \A_{\text{ineq}}\x_0, 
\end{equation}
where the additive constant, $\w^\top \x_0$, in the objective is omitted.
This is an inequality-form LP \eqref{prob:lp} with $\c = (\I - \A_{\text{eq}}^\dagger \A_{\text{eq}})^\top \w$, $\A = \A_{\text{ineq}}(\I - \A_{\text{eq}}^\dagger \A_{\text{eq}})$, and $\b = \b_{\text{ineq}} - \A_{\text{ineq}}\x_0$. 
Note that $\x = \zeros$ is always feasible for the resulting LP and that we can use this transformation if $\A_{\text{eq}}$ and $\b_{\text{eq}}$ are fixed, even if $\w$, $\A_{\text{ineq}}$, and $\b_{\text{ineq}}$ can change across instances.

\section{Derivation of the gradient in SGA}\label{app:gradient}
We explain how to derive the gradient of $u(\P, \pi)$ in \eqref{eq:gradient}, which indeed follows from \citet[Theorem~1]{Tan2020-lp} (or the implicit function theorem).
To focus on computing the gradient, we suppose that the columns of $\P$ have already been projected onto the feasible region of $\pi$ by the projection step~\eqref{eq:projection}.

Consider computing the gradient $\nabla u(\P, \pi)$ of $u(\P, \pi)=\max\Set*{ \c^\top \P\y}{\A\P\y \le \b }$ with respect to $\P$. 
For convenience, we define new parameters $\bm{w} = \P^\top\c\in\R^k$ and $\bm{W}=\A\P\in\R^{m\times k}$ and let $\y^*\in\argmax\Set*{\bm{w}^\top \y}{\bm{W}\y \le \b }$. 
Then, we can differentiate the optimal value, $u(\P, \pi) = \bm{w}^\top\y^*$, with respect to $\bm{w}$ and $\bm{W}$ by applying the implicit function theorem to the KKT condition. 
Specifically, as shown in \citet[Theorem~1]{Tan2020-lp}, we have $\frac{\partial u}{\partial \bm{w}} = \y^*$ and $\frac{\partial u}{\partial \bm{W}} = -\bm{\lambda}^*{\y^*}^\top$, where $\bm{\lambda}^*\in\mathbb{R}^m_{\ge0}$ is the dual optimal solution. 
From the chain rule, we have 
\[
  \nabla u(\P, \pi) = \frac{\partial u}{\partial \bm{w}} \cdot \frac{\partial \bm{w}}{\partial \P} + \frac{\partial u}{\partial \bm{W}} \cdot \frac{\partial \bm{W}}{\partial \P},
\] 
where the indices for the products are aligned appropriately.
By substituting the derivatives into this, we obtain $\nabla u(\P, \pi) = \c{\y^*}^\top - \A^\top\bm{\lambda}^*{\y^*}^\top$. When applying the implicit function theorem, we must ensure that the Jacobian matrix is invertible. In the above case, ensuring the regularity condition (i.e., active constraints are linearly independent at $\y^*$) is sufficient.

\section{Running time of learning methods}\label{app:time-learning}

\begin{figure*}[t]
  \centering
  \includegraphics[width=1.0\linewidth]{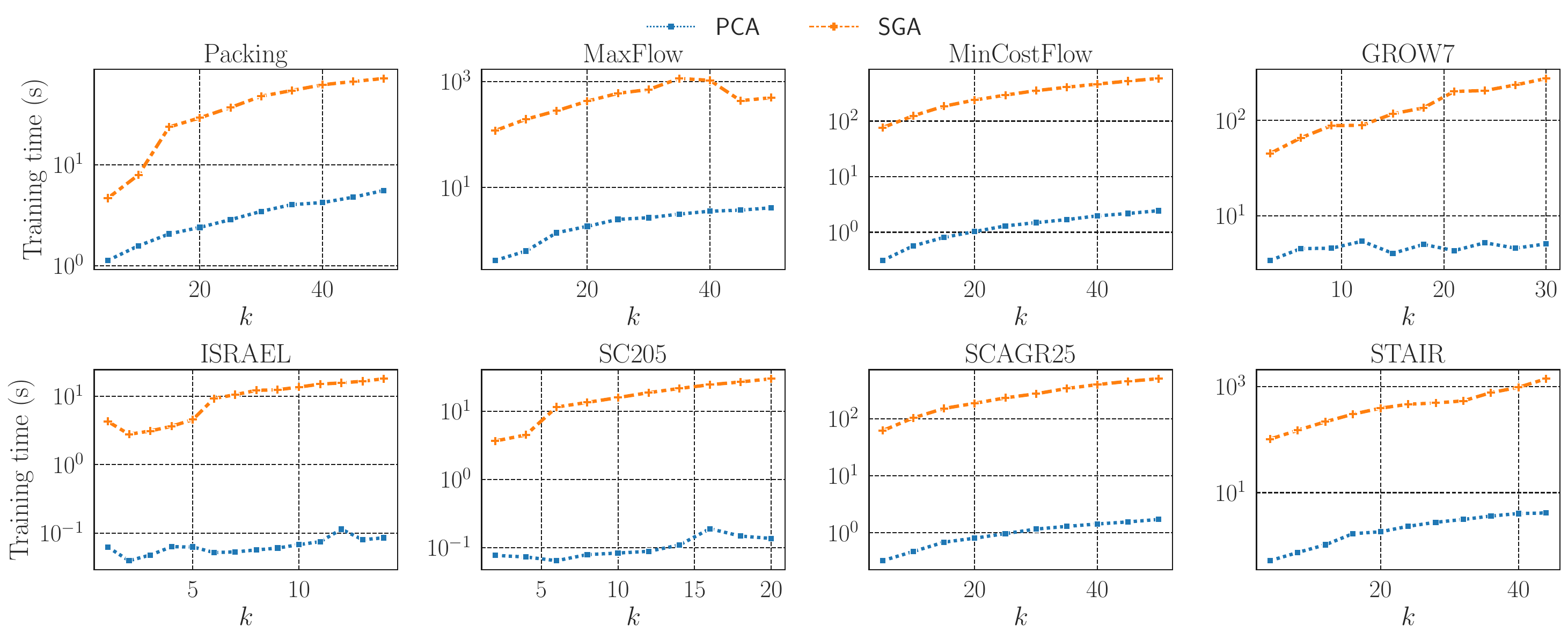}
  \caption{Running times of \PCA and \SGA for learning projection matrices on 200 training instances. 
  }
  \label{fig:training-time}
\end{figure*}

We discuss the theoretical complexity of the PCA- and SGA-based methods. 
For convenience, we use $T_\mathrm{lp}(m, n)$ to represent the time complexity of solving an LP instance with $m$ inequality constraints and $n$ variables, as this factor highly depends on problem settings.
Also, let $T_\mathrm{proj}(m)$ be the time for solving the problem in \eqref{eq:projection} $k$ times for projecting columns of $\P\in\R^{n\times k}$ onto the feasible region specified by $m$ inequality constraints. 
Given the $T_\mathrm{lp}(m, n)$-time linear optimization oracle, we can implement this projection step with a Frank--Wolfe-style algorithm. 
In this case, the time for projecting columns of $\P$ within an $\varepsilon$-error is typically $T_\mathrm{proj}(m) = T_\mathrm{lp}(m, n) \cdot {\rm poly}(n, m)\log(1/\varepsilon)$~\citep{Lacoste-Julien2015-ir,Garber2021-tg}.

\textbf{PCA-based method.}
Computing SVD of $\X \in \mathbb{R}^{N \times n}$ takes $\Ord(Nn^2)$ time. Then, the final projection takes up to $T_\mathrm{proj}(Nm)$ time. Thus, the total time complexity is $\Ord(Nn^2) + T_\mathrm{proj}(Nm)$.

\textbf{SGA-based method.} 
We discuss the complexity of a single iteration, which consists of projecting columns of $\P$ as in~\eqref{eq:projection}, solving a projected LP for obtaining $\y^*$ and $\bm{\lambda}^*$, and computing the gradient in~\eqref{eq:gradient}. 
These take $T_\mathrm{proj}(m)$, $T_\mathrm{lp}(m, k)$, and $\Ord(n(m + k))$ time, respectively. 
Thus, the per-iteration complexity is $T_\mathrm{proj}(m) + T_\mathrm{lp}(m, k) + \Ord(n(m + k))$. 
After finishing all the iterations, the final projection takes $T_\mathrm{proj}(Nm)$ time, as with the PCA-based method.
In the experiments, we ran \SGA for a single epoch, i.e., $N$ iterations. 
Thus, the total time complexity is $N(T_\mathrm{proj}(m) + T_\mathrm{lp}(m, k) + \Ord(n(m + k))) + T_\mathrm{proj}(Nm)$.

We turn to the running times of the PCA- and SGA-based methods in the experiments in \cref{section:experiment}.
\Cref{fig:training-time} shows the times taken by \PCA and \SGA for learning projection matrices on training datasets of 200 instances.
(\Full and \ColRand are not included since they do not learn projection matrices.)
We assumed that optimal solutions of training instances were computed a priori, and hence the time for solving original LPs was not included.
The figure shows that \SGA took much longer than \PCA. 
This is natural since \SGA iteratively solves LPs for computing gradients \eqref{eq:gradient} and quadratic programs for projection \eqref{eq:projection}, while \PCA only requires computing the top-($k-1$) right-singular vectors of $\X - \ones_N \bar\x^\top$, as discussed above.

\begin{figure*}[tb]
  \centering
  \includegraphics[width=.95\linewidth]{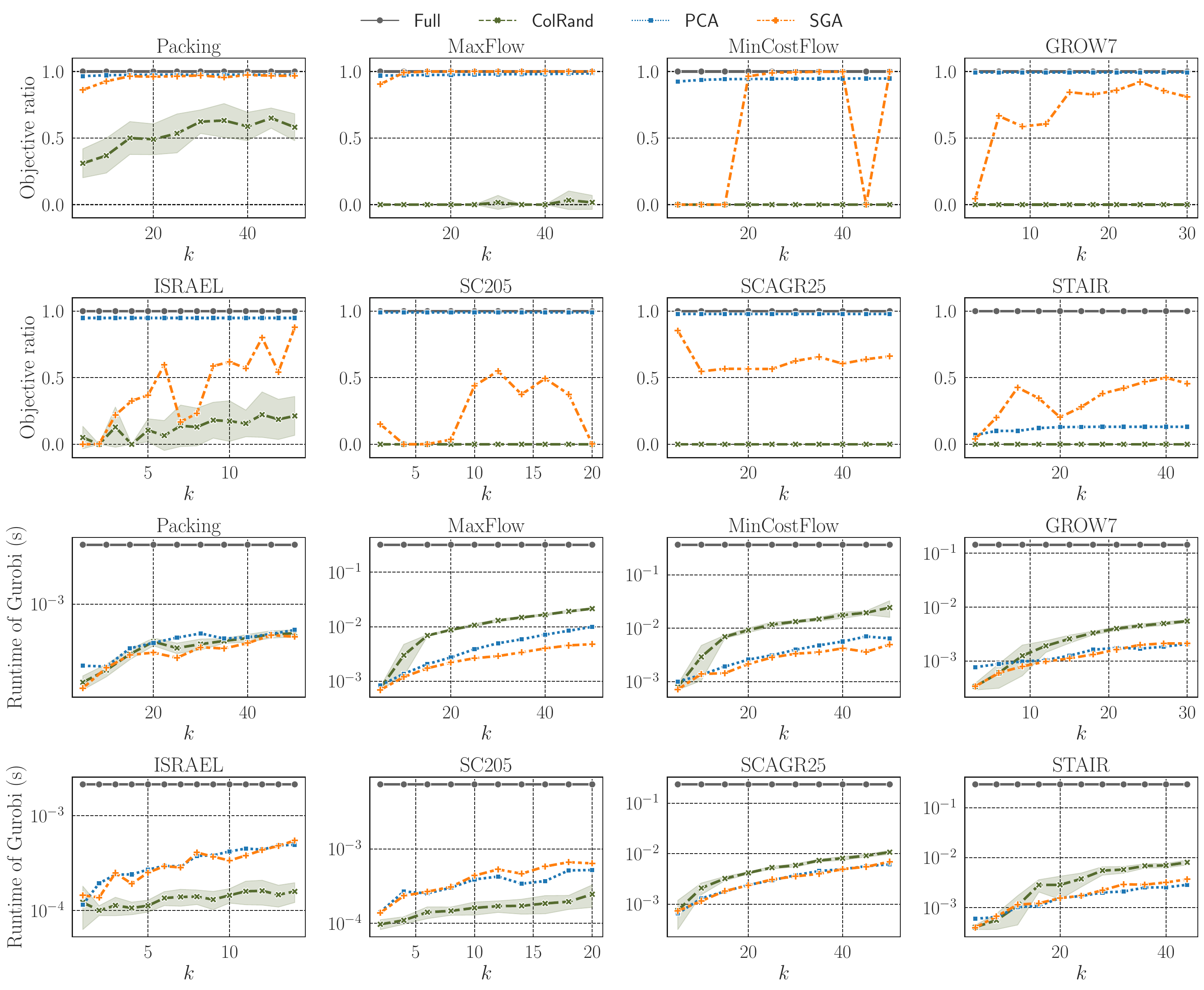}
  \caption{
    The same plots as in \cref{fig:results} but with \SGA initialized with \ColRand instead of \PCA, as mentioned in \cref{footnote:rand-init}.
    Compared with \cref{fig:results}, the objective ratio of \SGA deteriorates particularly in MinCostFlow, GROW7, SC205, and SCAGR25.
    }
  \label{fig:results_rand_init}
\end{figure*}

\end{document}